\documentclass{ecai} 

\usepackage{latexsym}
\usepackage{amssymb}
\usepackage{amsmath}
\usepackage{amsthm}
\usepackage{booktabs}
\usepackage{enumitem}
\usepackage{graphicx}
\usepackage{mathtools}
\usepackage{color}
\usepackage{hyperref}
\usepackage{algorithm}
\usepackage[noend]{algpseudocode}

\usepackage[capitalize,noabbrev]{cleveref}

\usepackage[switch]{lineno}

\usepackage{xspace}

\theoremstyle{plain}
\newtheorem{theorem}{Theorem}[section]

\theoremstyle{definition}
\newtheorem{definition}[theorem]{Definition}

\newtheorem{construction}{Construction} 
\theoremstyle{remark}

\newcommand{\BibTeX}{B\kern-.05em{\sc i\kern-.025em b}\kern-.08em\TeX}
\newcommand{\eg}{e.g.,\xspace}
\newcommand{\ie}{i.e.,\xspace}

\DeclareMathOperator*{\argmax}{arg\,max}

\makeatletter
\newcommand{\StatexIndent}[1][3]{%
  \setlength\@tempdima{\algorithmicindent}%
  \Statex\hskip\dimexpr#1\@tempdima\relax}
\makeatother

\newcommand{\A}{\mathcal{A}}
\newcommand{\D}{\mathcal{D}}
\newcommand{\F}{\mathcal{F}}
\newcommand{\E}{\mathbb{E}}
\newcommand{\G}{\mathcal{G}}

\newcommand{\M}{\mathcal{M}}
\renewcommand{\P}{\mathcal{P}}
\newcommand{\R}{\mathbb{R}}

\renewcommand{\S}{\mathcal{S}}
\newcommand{\T}{\mathcal{T}}

\newcommand{\SEQ}{\pi}
\newcommand{\POP}{\ensuremath{\overline{\pi}}\xspace}
\newcommand{\MaxPOP}{\ensuremath{\overline{\Pi}}\xspace}
\newcommand{\policy}{\rho}
\newcommand{\RM}{{\cal R}}
\newcommand{\MaxRM}{\ensuremath{{\cal R}_{\overline{\Pi}}}\xspace}

\begin{document}


\begin{frontmatter}


\paperid{414} 


\title{Maximally Permissive Reward Machines}


\author[A]{\fnms{Giovanni}~\snm{Varricchione}\thanks{Corresponding Author. Email: g.varricchione@uu.nl}}
\author[B,A]{\fnms{Natasha}~\snm{Alechina}}
\author[A]{\fnms{Mehdi}~\snm{Dastani}}
\author[C,A]{\fnms{Brian}~\snm{Logan}}

\address[A]{Utrecht University}
\address[B]{Open University}
\address[C]{University of Aberdeen}


\begin{abstract}
Reward machines allow the definition of rewards for temporally extended tasks and behaviors.
Specifying ``informative'' reward machines can be challenging. One way to address this is to generate reward machines from a high-level abstract description of the learning environment, using techniques such as AI planning. However, previous planning-based approaches generate a reward machine based on a single (sequential or partial-order) plan, and do not allow maximum flexibility to the learning agent. In this paper we propose a new approach to synthesising reward machines which is based on the set of partial order plans for a goal. We prove that learning using such ``\emph{maximally permissive}'' reward machines results in higher rewards than learning using RMs based on a single plan. We present experimental results which support our theoretical claims by showing that our approach obtains higher rewards than the single-plan approach in practice.
\end{abstract}

\end{frontmatter}


\section{Introduction}

Reward machines were introduced in \cite{Toro-Icarte//:18} as a way of defining temporally extended (\ie non-Markovian relative to the environment) tasks and behaviors. A \emph{reward machine} (RM) is a Mealy machine where states represent abstract `steps' or `phases' in a task, and transitions correspond to observations of \emph{high-level events} in the environment indicating that an abstract step/phase in the task has (or has not) been completed \cite{Camacho//:19,Toro-Icarte//:22}.
The RM-based algorithm proposed in \cite{Camacho//:19} has been shown to out-perform state-of-the-art RL algorithms, especially in tasks involving temporally extended behaviours. However, while learning with a reward machine is guaranteed to converge to an optimal policy \emph{with respect to the reward machine}, in general RMs provide no guarantees that the resulting policy is optimal \emph{with respect to the task} encoded by the reward machine. 
For example, a reward machine may specify that event $a$ should be observed before event $b$, while in some environment states, it may be more efficient to achieve $b$ before $a$. In general, for an RM-based policy to be optimal with respect to a task, the reward machine for the task must encode all possible ways the task can be achieved.

Another problem with reward machines is how to to generate them. While a declarative specification in terms of abstract steps or phases in a task is often easier to write than a conventional reward function, specifying a reward machine for a non-trivial task is challenging and prone to errors.
Reward machines can be computed from task specifications expressed in a range of goal and property specification languages, including LTL and LTL$_f$, in a straightforward way \cite{Camacho//:19}. However, reward machines generated from an abstract temporal formula may not expose significant task structure. Writing more ``informative'' specifications can be challenging, and, moreover, may inadvertently over-prescribe the order in which the steps are performed. 
One way to address this problem, is to generate a reward machine from a high-level abstract description of the learning environment, using techniques such as AI planning \cite{Illanes//:19,Illanes//:20}, or (in a multi-agent setting) ATL model checking \cite{Varricchione//:23}. 
For example, Illanes et al.\ \cite{Illanes//:19} consider a high-level model, in the form of a planning domain, of the environment in which the agent acts. They show how planning techniques can be used to synthesise a plan for a task, which is then used to generate a reward machine for that task. The reward machine is used to train a meta-controller for a hierarchical RL agent. The controller chooses which option (corresponding to an abstract action in the planning domain) to execute next. Their results indicate that an agent trained using a plan-based reward machine outperforms (is more sample efficient than) a standard HRL agent. They also show that reward machines based on partial-order plans outperform reward machines generated from sequential plans, arguing that this is because partial-order plans allow more ways of completing a task.

While the results presented by Illanes et al. are encouraging, their approach does not allow maximum flexibility to the agent, and thus cannot ensure learning an optimal policy for the task. The reward machine they generate is based on a single partial-order plan. In many cases, a goal may be achieved by different plans and each plan might be more appropriate in different circumstances, e.g., depending on the agent's location, or the resources available. 

In this paper we propose a new approach to synthesising reward machines which is based on the \emph{set of partial-order plans for a goal}. We present an algorithm which computes the set of partial-order plans for a planning task, and give a construction for synthesising a \emph{maximally permissive} reward machine (MPRM) from a set of partial-order plans. We prove that the expected discounted future reward of optimal policies learned using an MPRM is greater than or equal to that obtained from optimal policies learned using a RM synthesised from any single partial-order plan. We introduce a notion of the \emph{adequacy} of planning domain abstractions, which intuitively characterises when a planning domain captures all the relevant features of an MDP, and prove that the expected reward of an optimal policy learned using an MPRM synthesised from a \emph{goal-adequate} planning domain is the same as that of an optimal policy for the underlying MDP.
Finally, we evaluate MPRMs using three tasks in the \textsc{CraftWorld} environment \cite{Andreas//:17} used in \cite{Illanes//:19,Illanes//:20}, and show that the agent obtains higher reward than with RMs based either on a single partial-order plan or on a single sequential plan.

\section{Preliminaries}

In this section, we provide formal preliminaries for both reinforcement learning and planning. 

\subsection{Reinforcement Learning}

In RL the model in which agents act and learn is generally assumed to be a \emph{Markov Decision Process} (MDP) $M = \langle S, A, r, p, \gamma \rangle$, where $S$ is the set of states, $A$ is the set of actions, $r: S \times A \times S \to \R$ is the reward function, $p: S \times A \to \Delta(S)$ is the transition function, and $\gamma \in [0, 1]$ is the discount factor. It is assumed that the agent does not have access to the model in which it acts, i.e., $r$ and $p$ are hidden to it. The agent's goal in RL is to learn a \emph{policy} $\policy: S \to \Delta(A)$, i.e., a map from each state of the MDP to a probability distribution over the set of actions. In particular, we are mostly interested in so-called ``\emph{optimal policies}'', i.e., policies that maximise the expected discounted future reward from any state $s \in S$:
        
        \[ 
            \policy^* = \argmax_{\policy} \sum_{s \in S} v_\policy(s)
        \]

        where $v_\policy(s)$ is the ``\emph{value function}'', \ie the expected discounted future reward obtained from state $s$ by following policy $\policy$:

        \[
            v_\policy(s) = \E_\policy \left[ \sum_{t = 0}^\infty \gamma^t r_t \mid s_0 = s \right]
        \]

        where $r_t$ is the reward obtained at timestep $t$.
        
        As the MDP's dynamics and reward are hidden, the agent is supposed to learn a policy by trial and error. This is achieved by the agent taking an ``exploratory'' action $a$ in a state $s$, and observing which state $s'$ (sampled from $p(s, a)$) is reached and the reward $r' = r(s, a, s')$ that is obtained. By collecting these experiences $(s, a, s', r') \in S \times A \times S \times \R$, or ``\emph{samples}'', the agent can learn a policy $\policy$ via RL algorithms, such as \emph{Q-learning} \cite{Watkins//:92}.

    \subsection{Labelled MDPs}

    As in this work we assume the presence of planning domains and reward machines, we also assume that we are given a so-called ``\emph{labelling function}'' $L$. This function will be the link between the low-level MDP, in which agents learn how to act, and the high-level planning domain and reward machine, which describe how agents can achieve a task using high-level symbols and actions.

    \begin{definition}[Labelled MDP]
        Let $\P$ be a set of propositional symbols. Then, a \emph{labelled MDP} is a tuple $\M = \langle S, A, r, p, \gamma, L \rangle$, where $S, a, r, p$ and $\gamma$ are as in an MDP, and $L: S \to 2^\P$ is the labelling function, mapping each state of the MDP to a set of propositional symbols.
    \end{definition}

\subsection{Reward Machines}
    Reward machines \cite{Toro-Icarte//:22} are a tool recently introduced in the RL literature to define non-Markovian reward functions via finite state automata. Let $\M = \langle S, A, r, p, \gamma, L\rangle$ be a labelled MDP for some set of propositional symbols $\P$.

    \begin{definition}[Reward Machine]
        A reward machine (RM) is a tuple $\RM = \langle U, u_0, \Sigma, \delta_u, \delta_r \rangle$, where $U$ is the set of states of the RM, $u_0$ is the initial state, $\Sigma \subseteq 2^\P$ is the input alphabet, $\delta_u: U \times \Sigma \to U$ is the state transition function, and $\delta_r: U \times U \to \R$ is the reward transition function.
    \end{definition}

    When using RMs, training is usually done over the so-called ``\emph{product}'' between the labelled MDP and the RM, also known as a ``\emph{Markov Decision Process with a Reward Machine}'' (MDPRM) \cite{Toro-Icarte//:22}. 

    \begin{definition}[MDPRM]
        A \emph{Markov Decision Process with a Reward Machine} (MDPRM) is a tuple $\M = \langle S, A, p, \gamma, L, U, u_0, \delta_u, \delta_r \rangle$, where $S, A, p, \gamma, L$ are as in the definition of a labelled MDP, and $U, u_0, \Sigma, \delta_u$ and $\delta_r$ are as in the definition of a reward machine. 
    \end{definition}    

    At each timestep, the RM is in some state $u$. As the agent moves the MDP into state $s'$, the RM updates its internal state via the observation $L(s')$, \ie the new RM state is $u' = \delta_u(u, L(s'))$. Accordingly, the RM also outputs the reward $\delta_r(u, u')$, which is the reward the agent obtains. 
    As in ``vanilla'' MDPs, the agent learns a policy by taking exploratory actions and collecting rewards from the RM's reward function $\delta_r$. Thus, samples include also the states of the RM, i.e., each sample is a tuple $(s, u, a, s', u', r') \in S \times U \times \R \times S \times U$. For this reason, any RL algorithm that works with standard MDPs can also be used in MDPRMs. Moreover, algorithms exploiting access to the RM have also been proposed, e.g., CRM \cite{Toro-Icarte//:22}.

\subsection{Symbolic Planning}
\label{subsec:planning}

A \emph{planning domain} $\D = \langle \F, \A \rangle$, is a pair where $\F \subseteq \P$ is a set of \emph{fluents} (propositions), and $\A$ is a set of \emph{planning actions}. 
\emph{Planning states} are subsets $\S \subseteq \F$, where a proposition is in $\S$ if and only if it is true in $\S$. Actions $a \in \A$ are tuples $a = \langle \mathit{pre}^+, \mathit{pre}^-, \mathit{eff}^+, \mathit{eff}^- \rangle$ such that each element of $a$ is a subset of $\F$, $\mathit{pre}^+ \cap \mathit{pre}^- = \emptyset$ and $\mathit{eff}^+ \cap \mathit{eff}^- =\emptyset$. The ``$\mathit{pre}$'' sets are the sets of ``\emph{preconditions}'', whereas the ``$\mathit{eff}$'' are the sets of ``\emph{effects}'', or ``\emph{postconditions}''. $\mathit{pre}^+$ are the propositions that must be true to perform the action, whereas $\mathit{pre}^-$ those that must be false. Analogously, $\mathit{eff}^+$ are the propositions that are made true by the action, whereas $\mathit{eff}^-$ those that are made false. Thus, an action $a$ can be executed from a planning state $\S$ if and only if $\mathit{pre}^+ \subseteq \S$ and $\mathit{pre}^- \cap \S = \emptyset$.  Executing action $a$ in state $\S$ results in the new state $\S' = \left(\S \setminus \mathit{eff}^-\right) \cup \mathit{eff}^+$.
Given an MDP $\M$ and a planning domain $\D$, we assume that $A \cap \A = \emptyset$, \ie the planning actions are not the same as the actions the agent can perform in the MDP. Intuitively, the planning actions can be seen as \emph{high-level} or abstract actions which correspond to sequences of actions in the MDP. For example, in a Minecraft-like scenario where the agent can move up, down, left, and right on a grid, a planning action might be ``\emph{get wood}'' corresponding to a sequence of movement actions ending in a cell containing wood. 

As an example of a planning domain, consider the \textsc{CraftWorld} environment \cite{Andreas//:17} in which an agent moves in a grid and has to gather resources which can then be used to produce items.\footnote{\textsc{CraftWorld} is based on the popular video game Minecraft, and was used as a test domain in \cite{Illanes//:19,Illanes//:20}.} For example, the agent can build a bridge in order to reach the gold ore. We assume the agent  can build two different types of bridge: an iron bridge or a rope bridge. The iron bridge requires gathering wood and iron and then processing them  in a factory. The rope bridge requires gathering grass and wood and processing them in a toolshed. The corresponding planning domain $\D$ can be formalised as:
\begin{align*}
        \langle\F = \{ & \texttt{has-wood}, \texttt{has-grass}, \\&\texttt{has-iron}, \texttt{has-bridge} \},\\
        \A = \{ & \texttt{get-wood}, \texttt{get-grass}, \texttt{get-iron}, \\
                & \texttt{use-factory}, \texttt{use-toolshed} \}\rangle
\end{align*}
The \texttt{get-x} actions have no preconditions, and only one positive postcondition, \ie that \texttt{has-x} is true. The \texttt{use-factory} action has the preconditions \texttt{has-wood} and \texttt{has-iron}, the positive postcondition \texttt{has-bridge}, and the negative postconditions, \texttt{has-wood} and \texttt{has-iron}, \ie the \texttt{use-factory} action makes a bridge, ``consuming'' the resources collected by the agent in the process. The \texttt{use-toolshed} action has \texttt{has-wood} and \texttt{has-grass} as preconditions, the positive postcondition \texttt{has-bridge}, and the negative postconditions, \texttt{has-wood} and \texttt{has-grass}.

A \emph{planning task} is a triple $\T = \langle \D, \S_I, \G \rangle$, where $\D$ is a planning domain, $\S_I$ is the initial planning state, and $\G = \langle \G^+, \G^- \rangle$ is a pair containing two subsets of $\F$ which are disjoint. Any planning state $\S$ such that $\G^+ \subseteq \S$ and $\S \cap \G^- = \emptyset$ is a \emph{goal state}. 
For example, the planning task to build a bridge is given by the domain $\D$ we have previously defined, the initial state $\S_I = \emptyset$, and the goal $\G = \langle \{\texttt{has-bridge}\}, \emptyset \rangle$.

A \emph{sequential plan} $\SEQ = [a_0, \dots, a_n]$ for a planning task $\T$ is a sequence of planning actions $a_i \in \A$ such that: (i) it is possible to execute them sequentially starting from $\S_I$, and (ii) by doing so, the planning domain reaches a goal state.  For example, the following sequential plan allows the agent to produce a rope bridge:
\[
[\texttt{get-wood}, \texttt{get-grass}, \texttt{use-toolshed}]
\]

A \emph{partial-order plan} (POP) $\POP  = \langle \A', \prec \rangle$ is a pair where $\A'$ is a multiset of actions from $\A$ and $\prec$ is a partial order over $\A'$ \cite{Chapman:87a,McAllester/Rosenblitt:91a}.
We write $a \prec a'$ to denote $(a, a') \in \prec$, meaning that action $a$ must be performed before action $a'$. 
For example, the following partial-order plan allows the agent to produce an iron bridge: 
\begin{align*}\overline{\pi}&{}_{\texttt{iron-bridge}} = \\
&\langle \{ \texttt{get-wood}, \texttt{get-iron}, \texttt{use-factory} \}, \\
&\ \,\{\texttt{get-wood} \prec \texttt{use-factory}, \\
&\ \ \texttt{get-iron} \prec \texttt{use-factory} \} \rangle
\end{align*}
Sequential plans are a special case of partial-order plans where $\prec$ is a total order. 
In general, a partial-order plan corresponds to a \emph{set} of sequential plans, \ie the set of all sequential plans that can be obtained by extending the partial order $\prec$ to a total order (referred to as a ``\emph{linearisation}'' of the partial-order plan).
Compared to sequential plans, partial-order plans allow the agent greater flexibility in choosing the order in which actions are executed. While a sequential plan constrains the agent to follow the total order of the plan, with a partial-order plan the agent can perform any action $a$, so long as all actions $a'$ such that $a' \prec a$ have already been executed. 

Typically, given a planning task, a partial-order planner, \eg \cite{Nguyen/Kambhampati:01a}, returns a single partial-order plan $\POP = \langle \A', \prec \rangle$. 
However, in general, a planning task can be achieved using multiple partial-order plans, \ie plans $\POP'$ where the actions in $\A'$ are ordered differently, or which use different multisets of actions.

\begin{definition}[Set of all partial-order plans]
The \emph{set of all partial-order plans for a planning task} $\langle \D, \S_I, \G \rangle$, $\MaxPOP$, is the set of plans $\langle \A', \prec \rangle$ where $\A' \subseteq \A$ and any linearisation $[a_0, \dots, a_n]$ of $\A'$ consistent with $\prec$ results in a goal state $\S$, \ie $\G^+ \subseteq \S$ and $\G^- \cap \S = \emptyset$. 
\end{definition}

It is straightforward to give an algorithm that returns the set of all partial-order plans $\MaxPOP$  for a  planning task, see Algorithm \ref{alg:max-pop}. We assume the following definitions.
$steps(\POP)$ is the multiset of actions in the plan $\POP$ and $ord(\POP)$ is the set of ordering constraints. 
In addition, the algorithm maintains a set $links(\POP)$ of \emph{causal links} of the form $(a',p,a)$ where $a'$ and $a$ are steps and $p$ is a literal in the postcondition of $a'$ and in the precondition of $a$. Causal links record the reason for adding step $a'$ to the plan (in order to establish precondition of $a$), and are used to generate ordering constraints. A step $a''$ \emph{threatens} a causal link $(a', p, a)$ if $a''$ makes $p$ false.  To resolve the threat  $a''$ should be placed either before $a'$ in the order, or after $a$. An ordering is \emph{consistent} if it is transitive and does not contain cycles, \ie there is no $a_i, a_j$ such that $(a_i \prec a_j), (a_j \prec a_i) \in ord$.
Given a planning action $a$, $\mathit{pre}(a)$ is the set containing the positive and negative literals of the propositional symbols appearing in the sets $ \mathit{pre}^+$ and $\mathit{pre}^-$ of $a$, and $\mathit{eff}(a)$ is the set of  positive and negative literals in $ \mathit{eff}^+$ and $\mathit{eff}^-$. 
A precondition $p$ of a step $a$ is termed \emph{open} if there is no causal link $(a', p, a) \in links(\POP)$ establishing $p$. A plan is \emph{complete} if it has no open preconditions.
Initially, the set of plans is empty, and $\POP$ is initialised to a plan consisting of two steps: \emph{start} and \emph{finish}: \emph{start} has no preconditions and the initial state $\S_I$ as a postcondition; \emph{finish} has no postconditions and  the goal $\G$ as a precondition. $ord$ contains the single ordering constraint $\{ start \prec \mathit{finish} \}$, and $links$ is empty.

\begin{algorithm}[t]
\caption{Compute the set of all partial-order plans}
\label{alg:max-pop}
\begin{algorithmic}[1]
\State $\POP \gets \langle \{\mathit{start}, \mathit{finish}\}, \{\mathit{start} \prec \mathit{finish} \}\rangle$
\State $\MaxPOP \gets \emptyset$

\Procedure{pop-plan}{$\POP$}
\State $\mathit{open} \gets$ open preconditions $\in \mathit{steps}(\POP)$
\If{$\mathit{open} = \emptyset$}
\State $\MaxPOP \gets \MaxPOP \cup \{ \langle \mathit{steps}(\POP) \setminus  \{\mathit{start}, \mathit{finish}\}, \mathit{ord}(\POP) \setminus \{\mathit{start} \prec \mathit{finish} \} \rangle \}$
\Else
\For{$a \in \mathit{steps}(\POP)$ s.t. $p \in \mathit{pre}(a) \wedge p \in \mathit{open}$}
\For{$a' \in \A \cup \mathit{steps}(\POP)$ s.t. $p \in \mathit{eff}(a')$}
\If{$a'$ is new}
\State $\POP \gets \langle \mathit{steps}(\POP) \cup \{a'\}, $
\StatexIndent[6.8] $\mathit{ord}(\POP)\ \cup\ \{\mathit{start} \prec a' \prec \mathit{finish}\} \rangle$
\EndIf
\State $\mathit{ord}(\POP) \gets \mathit{ord}(\POP) \cup \{a' \prec a \}$
\State $\mathit{links}(\POP) \gets \mathit{links}(\POP) \cup \{ (a', p, a) \}$
\State \Call{order}{$\POP, a', p, a$}
\EndFor
\EndFor
\EndIf
\EndProcedure

\smallskip
\Procedure{order}{$\POP, a', p, a$}
\State $threats \gets \{(a_i, a_j) \mid (a_i, \neg p, a_j) \in links(\POP) \}$
\If{$threats \not= \emptyset$}
\State $cons \gets \{ \{o_1, \ldots, o_n\} \mid (a_j, a_k)_i \in threats\ \wedge\ $
\StatexIndent[5] $ o_i = a \prec a_j$ or $o_i = a_k \prec a' \}$
\For{$c \in cons$}
\If{$ord(\POP) \cup  c $ is consistent}
\State $ord(\POP) \gets ord(\POP) \cup  c$  
\State \Call{pop-plan}{$\POP$}
\EndIf
\EndFor
\Else
\State \Call{pop-plan}{$\POP$}
\EndIf
\EndProcedure
\end{algorithmic}
\end{algorithm}

The procedure \textsc{pop-plan} takes a partial-order plan $\POP$ as input. If $\POP$ has no open preconditions, \ie the plan is complete, then we remove the steps \emph{start} and \emph{finish}, add it to the set of plans, and \textsc{pop-plan} returns (lines 5-6). Otherwise, we iterate over each open precondition in the set of open preconditions, $open$ (lines 8-14). For each open precondition $p$, an action $a'$ from the set of actions $\A$ of the planning domain is chosen which establishes $p$ (line 9; if there are no actions which establish $p$, \ie the plan cannot be extended to a complete plan, this branch of the computation terminates and $\POP$ is discarded). The procedure \textsc{order} is then called (line 14) to resolve any threats introduced by the addition of $a'$. If there are no threats, then \textsc{pop-plan} is called again with the updated \POP containing $a'$ (lines 23-24). Instead, if there exists at least a threat, the set of sets of ordering constraints $cons$ (line 18) contains all possible ways of safeguarding each threatened link $(a_i, \neg p, a_j)$. For each such set of ordering constraints $c$, if $c$ is consistent with the current ordering constraints in $ord(\POP)$, they are added to $ord(\POP)$, and  \textsc{pop-plan} is called to extend the plan for each possible ordering of actions (lines 19-22). 
When a plan is found, we backtrack and continue from the `closest' enclosing \textbf{for} loop (which may be iterating over sets of ordering constraints in \textsc{order}, or actions $a'$ and open preconditions $p$ in \textsc{pop-plan}) to search for alternative ways of extending the incomplete plan $\POP$.
Algorithm \ref{alg:max-pop} runs in EXPSPACE, as the set of partial order plans is in the worst case exponential in the number of actions in $\A$. In practice, this is often not an issue: the planning domain is an abstraction of the underlying MDP, and the number of actions is typically small.

For the bridge task, the algorithm would produce another partial-order plan, in which the agent builds a rope bridge using grass:
\begin{align*}\overline{\pi}&{}_{\texttt{rope-bridge}}=\\
\langle \A = &\{ \texttt{get-wood}, \texttt{get-grass}, \texttt{use-toolshed} \}, \\
\ \,\prec = &\{\texttt{get-wood} \prec \texttt{use-toolshed},\\
\ &\texttt{get-grass} \prec \texttt{use-toolshed} \} \rangle
\end{align*}

Thus giving us the set of all partial-order plans for the bridge task: $\overline{\Pi}_{\texttt{bridge}} = \{ \overline{\pi}_{\texttt{iron-bridge}}, \overline{\pi}_{\texttt{rope-bridge}} \}$. In the Appendix, we also provide all sequential plans that can be obtained by linearising the POPs in $\overline{\Pi}_{\texttt{bridge}}$.

\section{Maximally Permissive Reward Machines}
\label{sec:mprm}

In this section, we show how the set of all partial-order plans, $\MaxPOP$, for a planning task $\T = \langle \D, \S_I, \G \rangle$, can be used to synthesise a reward machine $\MaxRM$ that is \emph{maximally permissive}, \ie which allows the agent maximum flexibility in learning a policy.

Let $\Pi$ be the set of all linearisations $\SEQ$ of all the partial-order plans in $\MaxPOP$. We denote by $\mathit{pref}(\SEQ)$ the set of all proper prefixes (of arbitrary length) of $\SEQ \in \Pi$. Note that the prefixes are finite, as the set of actions in $\MaxPOP$ is finite. Then, let $\mathit{states}(\mathit{pref}(\SEQ))$ be the set of sequences of planning states that is induced by the prefixes in $\mathit{pref}(\SEQ)$, assuming that the initial planning state is $\S_I$. We denote with $\mathit{steps}(\pi)$ the set of actions in a sequential plan, and with $\mathit{post}(\A')$ the set containing $\mathit{post}(a)$, as defined in Section \ref{subsec:planning}, for each planning action $a \in \A'$. For an arbitrary sequence of planning states $u$, we denote with $\mathit{last}(u)$ the last element of the sequence. For sets of literals $P$, we, respectively, denote with $P^+$ and $P^-$ the sets of propositional symbols with positive and negative literals in $P$.

\begin{construction}[Maximally Permissive Reward Machine (MPRM)]\label{constr:rm-from-mpp} \emph{Fix the set of all
partial-order plans $\MaxPOP = \{ \POP_1, \dots, \POP_n \}$ for some planning task $\T = \langle \D, \S_I, \G \rangle$. Then, $\MaxRM$, the maximally permissive RM  corresponding to \MaxPOP, is defined as follows:}
    \begin{itemize}
        \item $U = \left[ \bigcup_{\SEQ \in \Pi} \mathit{states}(\mathit{pref}(\SEQ)) \right] \cup \{ u_g \}$; 

        \item $u_0 = [\S_I]$;

        \item $\Sigma = \bigcup_{\SEQ \in \Pi} \ \mathit{post}(\mathit{steps}(\SEQ))$;

        \item $\delta_u(u, P) = u \S $, where $\S = \left(\mathit{last}(u) \setminus P^-\right) \cup P^+$ and $u\S \in \mathit{states}(\mathit{pref}(\SEQ))$ for some linearisation $\SEQ \in \Pi$, or $= u_g$ if $\G^+ \subseteq \S$ and $\G^- \cap \S = \emptyset$;

        \item $\delta_r(u, u') = \begin{cases}
            0 & \text{\emph{if} } u' = u_g \\
            -1 & \text{\emph{otherwise}}.
        \end{cases}$
    \end{itemize}
\end{construction}

In the construction of the MPRM, the set of states correspond to the set of all possible prefixes of planning states across all POPs in the set used to build the reward machine. Then, the RM transitions from a state $u$ to a state $u' = u\S$ when it observes the set of propositional symbols $P$ which are exactly the conditions such that $\S = \left(\mathit{last}(u) \setminus P^-\right) \cup P^+$ and, most importantly, $u\S$ is a prefix of some linearisation of a POP in $\MaxPOP$. As soon as the RM ``reaches'' a sequence of states such that the last state is a goal state for the task (meaning also that a linearisation has been ``completed''), it gives a reward of 0 to the agent and terminates in state $u_g$, while for all other transitions the agent gets a reward of $-1$. Note that if $u\S$ is not the prefix of any linearisation of a POP in $\MaxPOP$, or $\S$ is not a goal state, then $\delta_u(u, P) = u$.
Figure \ref{fig:MPRM-bridge} shows the MPRM synthesised from the set $\overline{\Pi}_{\texttt{bridge}}$ of all partial-order plans for the bridge example we gave in Section \ref{subsec:planning}.

\begin{figure}
    \centering
    \includegraphics[width=.45\textwidth]{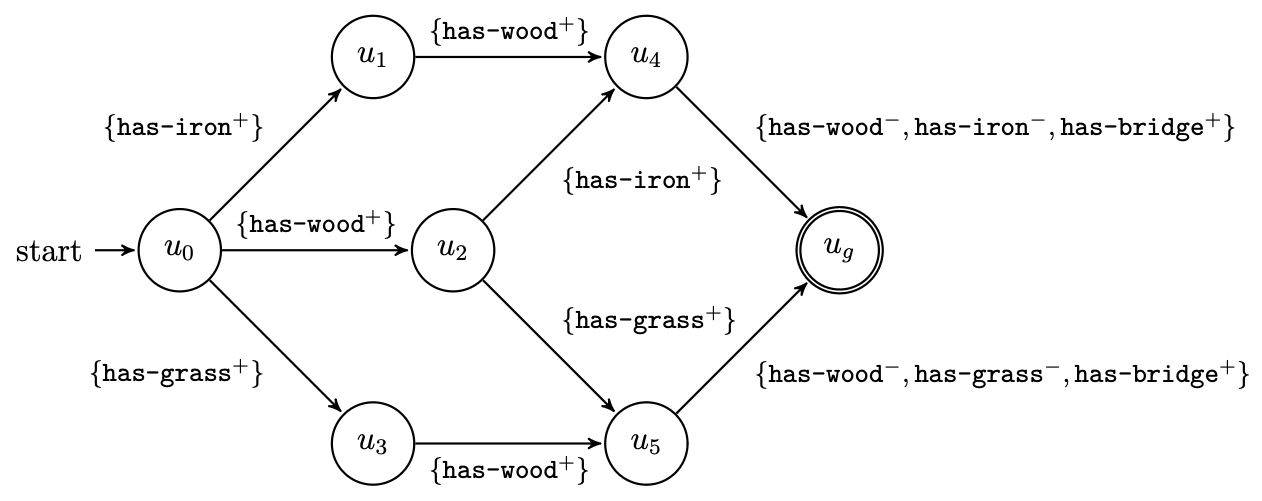}
    \caption{MPRM for the bridge task. Positive and negative postconditions are respectively denoted with a superscript $+$ and $-$.\vspace{.5cm}}
    \label{fig:MPRM-bridge}
\end{figure}

In the remainder of this section, we provide a theoretical analysis linking the optimal policies that can be learned by an agent depending on the kind of reward machine it is equipped with. We consider RMs that can be built from the set of all partial-order plans ($\text{RM-}\MaxPOP$), a single partial-order plan ($\text{RM-}\POP$), and a single sequential plan ($\text{RM-}\SEQ$), over the same planning domain $\D$. All RMs issue a non-negative reward only in the final state. We denote the optimal policy learnt using the set of all partial-order plans
by $\policy^*_{\text{RM-}\MaxPOP}$, using a single partial-order plan by $\policy^*_{\text{RM-}\POP}$, and using a single sequential plan by $\policy^*_{\text{RM-}\SEQ}$.

\begin{theorem}
\label{thm:theoretical-comparison-RMs}
  Let $\M$ be a labelled MDP, $\D$ a planning domain over $\M$, and $\text{RM-}\MaxPOP$, $\text{RM-}\POP$ and $\text{RM-}\SEQ$ final state reward machines generated from $\D$ for the same task. Then,
    \[
        \policy^*_{\text{RM-}\MaxPOP} \geq \policy^*_{\text{RM-}\POP} \geq \policy^*_{\text{RM-}\SEQ}
    \]
    where 
    $\policy_1 \geq \policy_2$ if and only if $v(\policy_1(s)) \geq v(\policy_2(s))$ for all states $s \in S$ of $\M$.
\end{theorem}

\begin{proof}
    The proof follows from the fact that any policy that can be learned using an RM synthesised from a single sequential plan can also be learned using an RM synthesised from a partial-order plan which has the sequential plan as its linearisation (if it remains an optimal policy). Similarly, a policy learned using an RM synthesised from a single partial-order plan can also be learned using an RM that is synthesised from the set of all partial-order plans.  
\end{proof}

Theorem \ref{thm:theoretical-comparison-RMs} shows that MPRMs allow an agent to learn an optimal policy with respect to the planning domain and the planning task. A natural question to ask is whether it learns a \emph{goal-optimal} policy $\policy^*$, \ie a policy that achieves the goal using the smallest number of actions in the underlying MDP. For example, an agent using Q-learning is guaranteed to learn a goal-optimal policy on an MDP where the agent is always given the same negative reward and the discount factor $\gamma$ is exactly 1. 

An agent using an MPRM will learn a goal-optimal policy if the planning domain and labelling are ``\emph{adequate}'' for the goal.
    
\begin{definition}\label{def:goal-adequate}
Given a labelled MDP, planning domain $\D$, and goal $\G$, we say that $\D$ is \emph{adequate} for $\G$ if, and only if:
\begin{itemize}
\item $\G$ corresponds to a set of planning domain fluents, \ie $\G \subseteq \F$; 
\item a goal-optimal policy encounters all the state labels in some plan $\POP \in \MaxPOP$ for $\G$, in the
order consistent with the order in $\POP$. 
\end{itemize}
\end{definition}

For example, if any policy to build a bridge has to encounter labels corresponding to getting wood, getting iron and using a factory, a planning domain and labelling containing only these fluents is adequate for the goal of having a bridge. However, if there is an alternative way of building a bridge that involves getting grass, and this label is missing in the planning domain, then the domain is not adequate for the goal of building a bridge.

\begin{theorem}\label{thm:optimal}
$\policy^* = \policy^*_{\text{RM-}\MaxPOP}$ if ${\text{RM-}\MaxPOP}$ is synthesized from a goal-adequate planning domain.
\end{theorem}
\begin{proof}
From the definition of a planning domain adequate for the goal, any goal-optimal policy has to go through the way-points encoded in the reward machine.
\end{proof}
\section{Empirical Evaluation}\label{sec:experiments}

In this section, we evaluate maximally permissive reward machines in three tasks in the \textsc{CraftWorld} environment, and show that the agent obtains higher reward with an MPRM than with RMs based on a single partial-order plan or a single sequential plan. In the first task, the agent has to build a bridge, as in the example in Section  \ref{subsec:planning} For the second task, the agent has to collect gold. In the third  task, the agent has to collect gold or a gem, and the task is considered achieved when the agent collects at least one of the two items. For the gold-or-gem task we have to slightly modify the definition of goal states in planning tasks: the goal is the pair $\G = \langle \G^+ = \{ \texttt{has-gold}, \texttt{has-gem} \}, \G^- = \emptyset \rangle$, and a planning state $\S$ is a \emph{goal} state if and only if $\G^+ \cap \S \neq \emptyset$ and $\G^- \cap \S = \emptyset$. The gold and the gem are collected as described in \cite{Andreas//:17}: gold is collected by using a(ny) bridge, whereas the gem is collected using an axe. To produce an axe, the agent must combine a stick, which can be obtained by processing wood at the workbench, with iron at the toolshed. We refer to these, respectively, as the ``\emph{bridge task}'', ``\emph{gold task}'', and ``\emph{gold-or-gem task}''. In the planning domain, we add the following fluents: 
\begin{itemize}
    \item For the gold task: \texttt{has-gold};
    \item For the gold-or-gem task: \texttt{has-gold}, \texttt{has-stick}, \texttt{has-axe}, \texttt{has-gem};
\end{itemize}
and the following planning actions:
    \begin{itemize}
    \item For the gold task: 
        \begin{itemize}
            \item \texttt{get-gold}, with one positive precondition, \texttt{has-bridge}, and one positive postcondition, \texttt{has-gold};
        \end{itemize}
    \item For the gold-or-gem task: 
        \begin{itemize}
            \item \texttt{use-workbench}, with one positive precondition, \texttt{has-wood}, one positive postcondition, \texttt{has-stick}, and one negative postcondition, \texttt{has-wood};
            \item \texttt{use-toolshed-for-axe}, with positive preconditions, \texttt{has-stick} and \texttt{has-iron}, one positive postcondition, \texttt{has-axe}, and two negative postcondition, \texttt{has-stick} and \texttt{has-iron};
            \item \texttt{get-gem}, with one positive precondition, \texttt{has-axe}, and one positive postcondition, \texttt{has-gem};
        \end{itemize}
\end{itemize}

The set of partial-order plans for the bridge task $\overline{\Pi}_{\texttt{bridge}} = \{ \overline{\pi}_{\texttt{iron-bridge}}, \overline{\pi}_{\texttt{rope-bridge}} \}$ is given in Section \ref{subsec:planning}. For the gold task, we extend $\overline{\pi}_{\texttt{iron-bridge}}$ and $\overline{\pi}_{\texttt{rope-bridge}}$ by adding the \texttt{get-gold} action, and by having $\texttt{use-factory} \prec \texttt{get-gold}$ and $\texttt{use-toolshed} \prec \texttt{get-gold}$. For the gold-or-gem task, the set of partial-order plans consists of  $\overline{\pi}_{\texttt{iron-bridge}}$, $\overline{\pi}_{\texttt{rope-bridge}}$ and $\overline{\pi}_{\texttt{gem}}$ in which the agent makes an axe and uses it to mine the gem. $\overline{\pi}_{\texttt{gem}}$ is defined as follows:
\begin{align*}
\overline{\pi}_{\texttt{gem}} &=\\ 
    \langle \A = \{ &\texttt{get-wood}, \texttt{get-iron}, \texttt{use-workbench},\\ 
            &\texttt{use-toolshed-for-axe}, \texttt{get-gem} \},\\
    \prec = \{  &\texttt{get-wood} \prec \texttt{use-workbench}, \\
                &\texttt{get-iron} \prec \texttt{use-toolshed-for-axe}, \\
                &\texttt{use-workbench} \prec \\&\qquad\texttt{use-toolshed-for-axe}, \\
                &\texttt{use-toolshed-for-axe} \prec \texttt{get-gem}\}\rangle
\end{align*}

For each task, we also generate all sequential plans that can be obtained by linearising the POPs that can be used to achieve the task. 
Thus, for both the bridge and gold tasks, there are a total of 4 sequential plans and 2 partial-order plans. For the gold-or-gem task, there are a total of 7 sequential plans and 3 partial-order plans. In the Appendix\footnote{The extended version of this paper with the appendix can be found on arXiv.}, we provide formal definitions of the planning domains, and give also the plans for the gold and gold-or-gem tasks.

\subsection{Experimental Setup}

The maximally permissive RMs for each task were synthesised using the construction given in Section \ref{sec:mprm}. The RMs for each partial-order and sequential plan were generated using the approach presented in \cite{Illanes//:19}. Training is carried out by using Q-learning over the resulting MDPRMs \cite{Toro-Icarte//:22}.\footnote{Note that we do not provide results for a baseline that does not employ reward machines (\eg Q-learning): as shown in \cite{Toro-Icarte//:22}, \textsc{CraftWorld} is a complex environment with sparse rewards, making it infeasible for an agent to learn an effective policy without having access to a reward machine.} 

For each task we generated 10 different maps of size 41 by 41 cells. The maps and initial locations were chosen so that from some locations a task can be completed more quickly by following a particular sequential plan. For example, in the first map for the bridge task, if the agent starts from a location in the upper half of the map (\ie in the first 20 rows) it is more convenient to build an iron bridge, while in the lower half of the map it is more convenient to build a rope bridge. The MDP reward function $r$ returns $-$1 for each step taken by the agent, until it achieves the task or the episode terminates. When the task is completed, the map and agent are ``re-initialised'':  the agent is placed on a random starting cell and its ``inventory'' is emptied, \ie it contains no items. For each set of plans and single partial-order/sequential plan for a task, and for each of the 10 maps for the task, an agent was trained with the corresponding RM for 10,000,000 training steps.
Training was carried out in episodes, lasting at most 1,000 steps, after which the environment was re-initialised regardless of whether the agent has achieved the task or not. Every 10,000 training steps the agent was evaluated on the same map used for training from 5 (predetermined) starting positions. We set the learning rate $\alpha =0.95$, the discount rate $\gamma = 1$, and the exploration rate $\varepsilon = 0.1$. 

Our implementation, largely based off of that of \cite{Illanes//:19}, is available in the following \href{https://github.com/giovannivarr/MPRM-ECAI24}{GitHub repository}.

\subsection{Results}

For each approach, we plot the median and the 25$^{\text{th}}$ and the 75$^{\text{th}}$ percentiles (shaded areas) of the rewards obtained across all maps by the agents in the evaluations during training for each task. To make the plots more readable, we have ``aggregated'' results for the sequential and partial-order plan-based RMs: for each kind of plan we plot the median and the 25$^{\text{th}}$ and the 75$^{\text{th}}$ percentiles of all agents trained with an RM generated using that type of plan. In the plots, ``\texttt{QRM-MPRM}'' denotes the performance of the agent trained with a maximally permissive RM, while ``\texttt{Aggregated-QRM-Seq}'' and ``\texttt{Aggregated-QRM-POP}'' are, respectively, the aggregated performance of agents trained with sequential plan RMs and partial-order plan RMs. On the $x$-axis we plot the number of steps (in millions), while on the $y$-axis we plot the performance obtained during the evaluations run at the corresponding timestep.

\begin{figure}[ht]
    \begin{center}
    \centerline{\includegraphics[width=.45\textwidth]{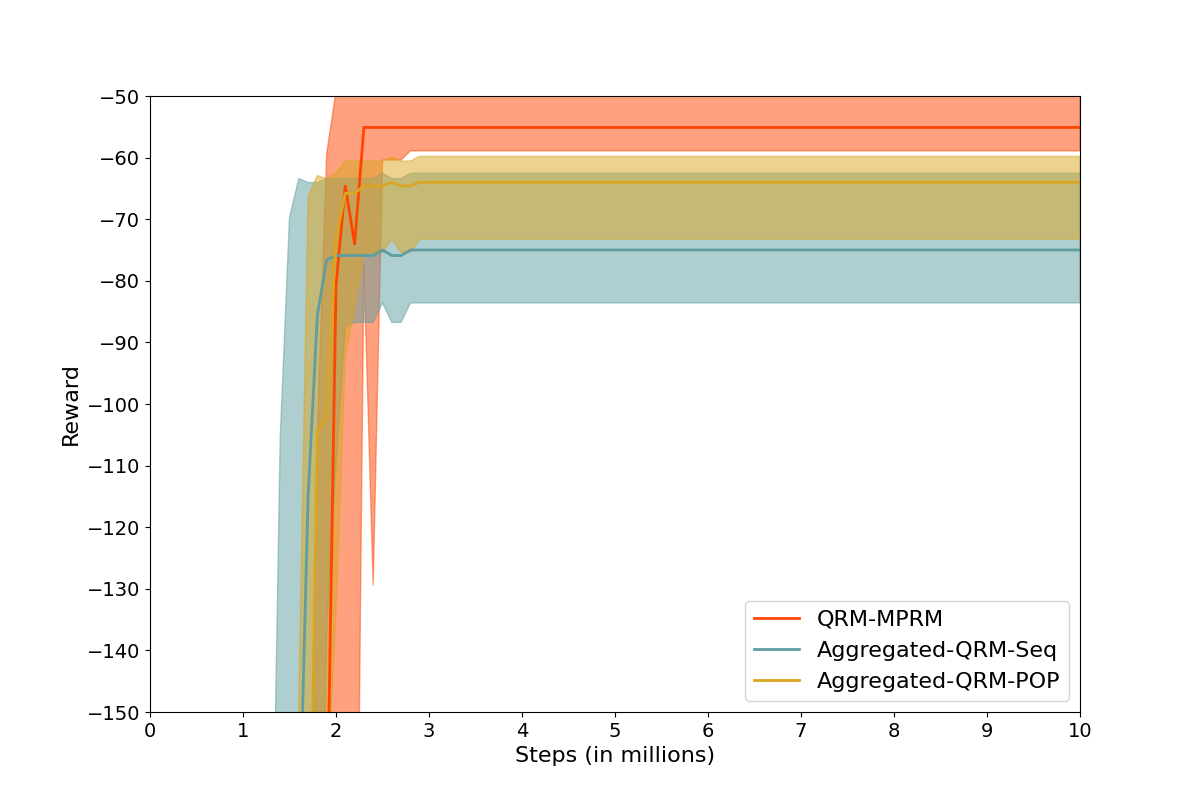}}
    \caption{Results for the bridge task.}
    \label{fig:bridge-qrm}
    \end{center}
\end{figure}

\begin{figure}[ht]
    \begin{center}
    \centerline{\includegraphics[width=.45\textwidth]{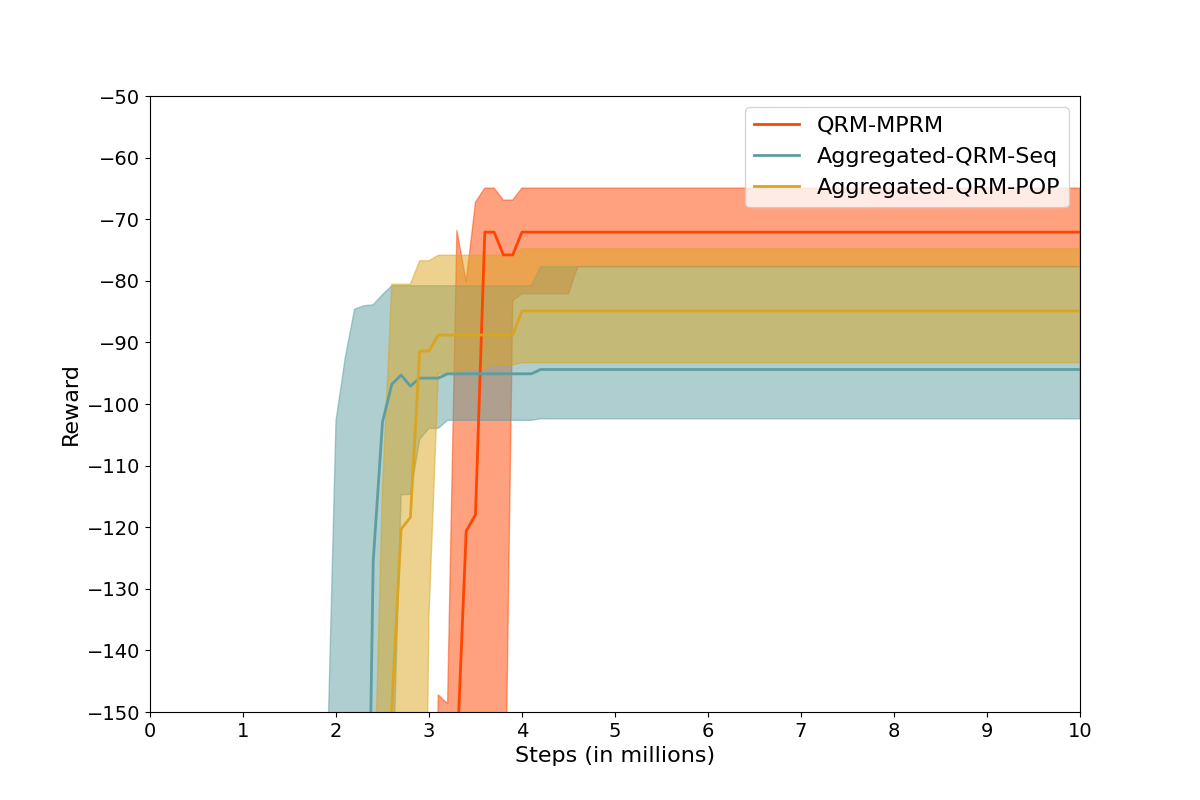}}
    \caption{Results for the gold task.}
    \label{fig:gold-qrm}
    \end{center}
\end{figure}

\begin{figure}[ht]
    \begin{center}
    \centerline{\includegraphics[width=.45\textwidth]{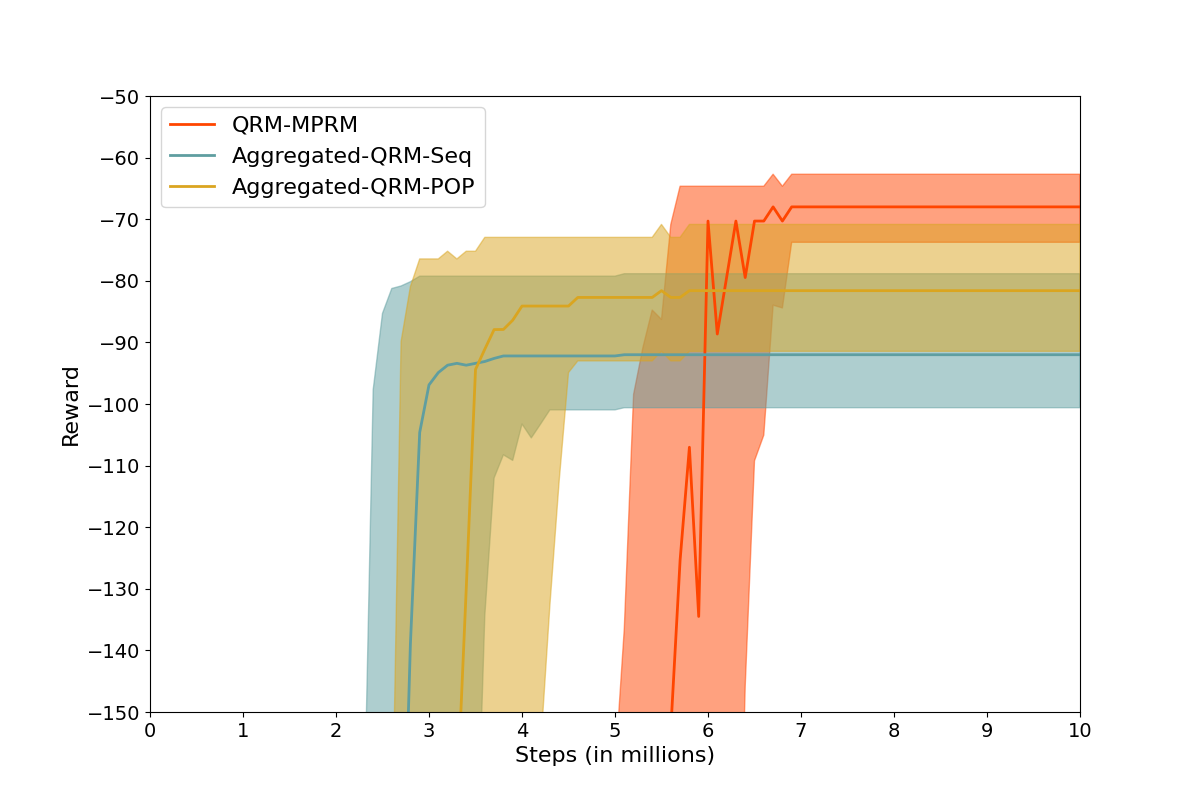}}
    \caption{Results for the gold-or-gem task.}
    \label{fig:gold-gem-qrm}
    \end{center}
\end{figure}

\begin{figure*}[t]
    \centering
    \includegraphics[width=0.7\textwidth, height=0.42\textheight]{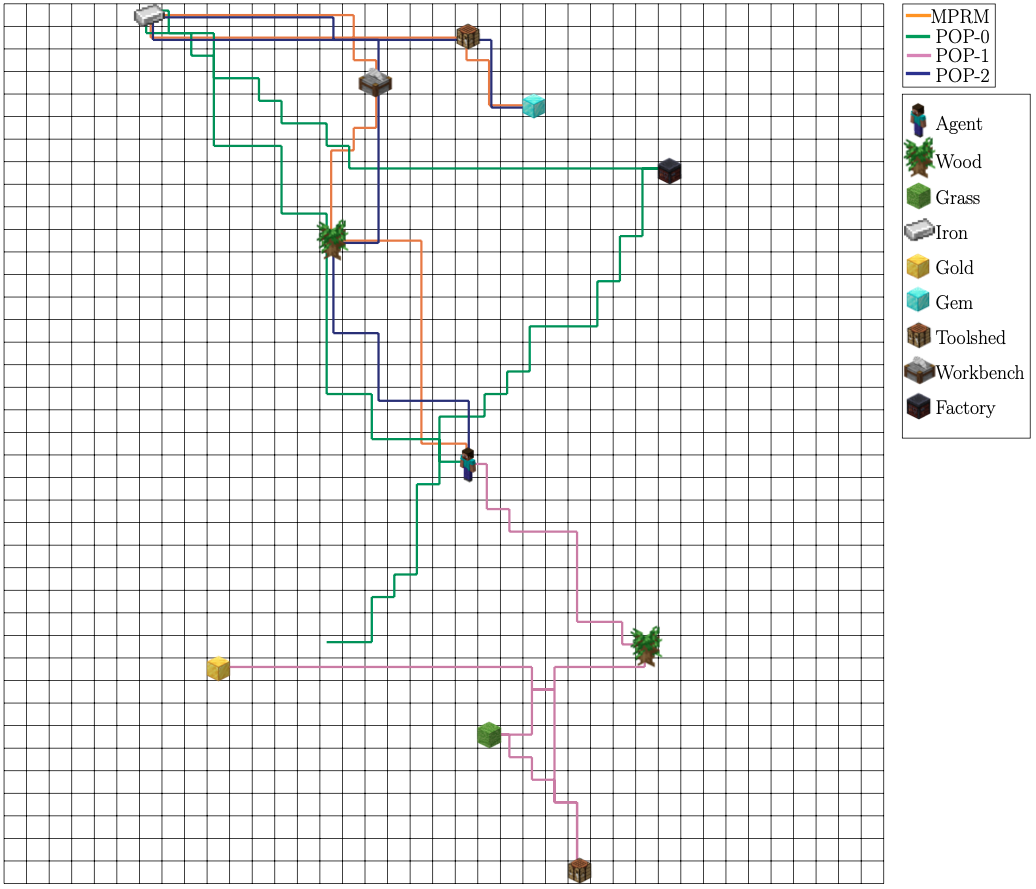}
    \caption{Illustration of the behaviour of the MPRM and POP-trained agents on the gold-or-gem task.}
    \label{fig:demo}
    \bigskip
\end{figure*}

Figure \ref{fig:bridge-qrm} shows the results for the bridge task, Figure \ref{fig:gold-qrm} shows the results for the gold task, and Figure \ref{fig:gold-gem-qrm} shows the results for the gold-or-gem task. As can be seen, in all tasks the agent trained with the MPRM outperforms the aggregated results for the agents trained using RMs based on a single partial-order or sequential plan. This is as expected given Theorem \ref{thm:theoretical-comparison-RMs}. In addition, the agent trained using an RM based on a single partial-order plan outperforms the agent trained using an RM based on a single sequential plan. Again, this is consistent with Theorem \ref{thm:theoretical-comparison-RMs} and the results in \cite{Illanes//:19}. However, in all experiments, the agent trained with the maximally permissive RM converges more slowly than the other agents (particularly in the gold-or-gem task). Intuitively, increasing the flexibility of the RM trades solution quality for sample complexity. Note that, as the planning domain used is \emph{adequate} in the sense defined in Section \ref{sec:mprm} for all the tasks, the MPRM agent can learn an optimal policy for each task.

Figure \ref{fig:demo} illustrates the behaviour of the agent trained using the MPRM and the agents trained using the RMs generated from each of the partial-order plans on a map for the gold-or-gem task. For readability, we have decided to not include agents trained with a sequential plan-based RM: note that none of them achieved the task in fewer steps than the agents shown in the figure. In the supplementary material we provide, for each agent, a file containing its trajectory (\ie the ordered sequence of coordinates of the cells it visited) in the test, also for agents trained with a sequential plan-based RM. Given the initial position of the agent, the optimal plan for the task is to collect a gem (POP-2). As can be seen, both the agent trained using the POP-2-based RM and the MPRM agent achieve the goal by collecting a gem and complete the task in 60 steps. The agent trained using the RM based on the POP to collect gold using a rope bridge (POP-1) is also able to achieve the task, but in 63 steps. However, the agent trained using the RM based on the POP to collect gold using an iron bridge (POP-0) is unable to complete the task after 10,000,000 timesteps. This illustrates that the MPRM agent is able to choose the ``correct'' plan for the agent's position, and the inherent problems of training agents using an RM based on a single plan.
\section{Related Work}

Reward machines have been used both in single-agent RL \cite{Toro-Icarte//:18,Dohmen//:22,Corazza//:22,Furelos//:23} and multi-agent RL \cite{Neary//:21,Hu//:24}. As mentioned in the Introduction, approaches to synthesise reward machines from high-level specifications have also been proposed; however, to the best of our knowledge, only \cite{Illanes//:19,Illanes//:20} and our work generate reward machines from plans.

Another line of research focuses on learning reward machines from experience. In \cite{Toro-Icarte//:19} an approach is proposed that uses Tabu search to update the RM hypothesis, by searching through the trace data generated by the agent exploring the environment. \cite{Xu//:20} presented an approach where the RM hypothesis is updated whenever traces that are inconsistent with the current one are detected. In \cite{Hasanbeig//:21}  RM-learning is reduced to SAT solving, and solved using DPLL. While these approaches do not require an abstract model of the environment in the form of a planning domain, they focus on learning a reward machine for a single task. In \cite{Toro-Icarte//:19} and \cite{Xu//:20} the agent learns a policy for each state of the RM hypothesis. However, when the latter is updated, it has to re-learn such policies from scratch (\cite{Xu//:20} tries to mitigate this issue by transferring a subset of the policies, but this is not always possible). 
Moreover, all these approaches assume that the agent is able to generate ``positive'' traces, \ie traces in which the task is achieved. While in simple environments this is a reasonable assumption, for more complex environments with sparse rewards it may be difficult to generate positive traces. 

Planning has been applied to reinforcement learning since at least \cite{Grounds//:05}, which combined Q-learning with STRIPS planning. More recently, \cite{Yang//:18} proposed an approach integrating planning with options \cite{Sutton//:99,Dietterich//:00}. In \cite{Lyu//:19} a framework is introduced that exploits planning to improve sample efficiency in deep RL.
In both of these approaches the RL experience is then used to improve the planning domain, similarly to what happens in model-based RL. Then, the new plan obtained using the updated domain is used to train again the RL agent. 
In \cite{Schrittwieser//:20,Jin//:22,Wu//:22} abstract models for the low-level MDP and/or its actions are learned so that planning can be leveraged to improve learning. However, all of these approaches assume that learning is guided by a single (sequential) plan.

\section{Conclusions}

We have proposed a new planning-based approach to synthesising maximally permissive reward machines which uses the set of partial-order plans for a goal rather than a single sequential or partial-order plan as in previous work.  Planning-based approaches have the advantage that it is straightforward to train agents to achieve new tasks ---  given a planning domain, we can automatically generate a reward machine for a new task. We have provided theoretical results showing how agents trained using maximally permissive reward machines learn policies that are at least as good as those learned by agents trained with a reward machine built from an individual sequential or partial-order plan, and the expected reward of an optimal policy learned using an MPRM synthesised from a \emph{goal-adequate} planning domain is the same as that of an optimal policy for the underlying MDP. Experimental results from three different tasks in the \textsc{CraftWorld} environment suggest that these theoretical results apply in practice. However, our results also show that agents trained with maximally permissive RMs converge more slowly than agents trained using RMs based on a single plan. We believe this is because the increased flexibility of maximally permissive RMs trades solution quality for sample complexity. Our approach is therefore most useful when the quality of the resulting policy is paramount. 

A limitation of our approach is that, in the worst case, the set of all partial order plans for a task may be exponential in the number of actions in the planning domain. In future work we would like to investigate the use of \emph{top-k} planning techniques, \eg \cite{Katz/Lee:03a}, to sample a diverse subset of the set of all plans. Intuitively, such an approach could allow the quality of the resulting policy to be traded off against the number of plans in the sample.

Another line of future work is to investigate option-based approaches to learning \cite{Sutton//:99,Dietterich//:00} as in \cite{Illanes//:20}, where each abstract action in a plan is ``implemented'' as an option. We expect results similar to the ones in this paper, where the agent trained with all partial-order plans is able to achieve a better policy but converging slower.

Finally, the experiments in Section \ref{sec:experiments} are limited to discrete environments. However, our approach is applicable to environments with continuous action and state spaces. Reward machines have previously been successfully applied in such environments \cite{Toro-Icarte//:22,Furelos//:23}, and planning domains, which form the basis our approach, are agnostic about the underlying environment, as they are defined in terms of states resulting from (sequences of) MDP actions rather than the actions themselves. Nevertheless, learning in continuous environments is more challenging than learning in discrete ones, and evaluating the benefits of our approach in such environments is future work.

\bibliography{ecai24-mpprl}

\clearpage
\section*{Appendix}

\subsection*{Sequential plans for the bridge task}
    Recall that the set of POPs for the bridge task is $\overline{\Pi}_{\texttt{bridge}} = \{ \overline{\pi}_{\texttt{rope-bridge}}, \overline{\pi}_{\texttt{iron-bridge}} \}$. By linearising them, we obtain the following sequential plans:

    \begin{itemize}
        \item[] $\pi_0 = [\texttt{get-wood}, \texttt{get-iron}, \texttt{use-factory}]$
        \item[] $\pi_1 = [\texttt{get-wood}, \texttt{get-grass}, \texttt{use-toolshed}]$
        \item[] $\pi_2 = [\texttt{get-iron}, \texttt{get-wood}, \texttt{use-factory}]$
        \item[] $\pi_3 = [\texttt{get-grass}, \texttt{get-wood}, \texttt{use-toolshed}]$
    \end{itemize}

\subsection*{Planning domain and plans for the gold task}
The planning domain $\D_{\texttt{gold}} = \langle \F, \A \rangle$ is the following:

\begin{itemize}
        \item[] $\begin{aligned}\F = \{ & \texttt{has-wood}, \texttt{has-grass}, \\
                &\texttt{has-iron}, \texttt{has-bridge},\\
                & \texttt{has-gold} \}\end{aligned}$
        \item[] $\begin{aligned}\A = \{ & \texttt{get-wood}, \texttt{get-grass}, \texttt{get-iron}, \\
                & \texttt{use-factory}, \texttt{use-toolshed}, \\
                & \texttt{get-gold}\}\end{aligned}$
\end{itemize}

where each action has the following pre- and postconditions (recall that each tuple includes, in order, the set of positive and negative preconditions, and the set of positive and negative postconditions of the action):
\begin{itemize}
    \item[] $\texttt{get-wood} = \langle \emptyset, \emptyset, \{ \texttt{has-wood}\}, \emptyset \rangle $
    \item[] $\texttt{get-grass} = \langle \emptyset, \emptyset, \{ \texttt{has-grass}\}, \emptyset \rangle $
    \item[] $\texttt{get-iron} = \langle \emptyset, \emptyset, \{ \texttt{has-iron}\}, \emptyset \rangle $
    \item[] $\begin{aligned}\texttt{use}&\texttt{-factory} =\\&\langle \{ \texttt{has-wood}, \texttt{has-iron}\}, \emptyset,\\&\{ \texttt{has-bridge}\}, \{ \texttt{has-wood}, \texttt{has-iron} \} \rangle\end{aligned}$
    \item[] $\begin{aligned}\texttt{use}&\texttt{-toolshed} = \\&\langle \{ \texttt{has-wood}, \texttt{has-grass}\}, \emptyset,\\& \{ \texttt{has-bridge}\}, \{ \texttt{has-wood}, \texttt{has-grass} \} \rangle \end{aligned}$
    \item[] $\texttt{get-gold} = \langle \{\texttt{has-bridge} \}, \emptyset, \{ \texttt{has-gold}\}, \emptyset \rangle $
\end{itemize}

For the planning task we have $\S_I = \emptyset$, and $\G = \langle \{ \texttt{has-gold} \}, \emptyset \rangle$. The set of POPs $\overline{\Pi}_{\texttt{gold}}$ for this task is the following:
\begin{itemize}
    \item[] $\begin{aligned}\overline{\pi}_{\texttt{iron-gold}} &= \\
    \langle \A = \{ &\texttt{get-wood}, \texttt{get-iron}, \\
            &\texttt{use-factory}, \texttt{get-gold}\},\\
    \prec = \{  &\texttt{get-wood} \prec \texttt{use-factory}, \\
                &\texttt{get-iron} \prec \texttt{use-factory}, \\
                &\texttt{use-factory} \prec \texttt{get-gold}\}\rangle\end{aligned}$
    \item[] $\begin{aligned}\overline{\pi}_{\texttt{rope-gold}} &= \\
    \langle \A = \{ &\texttt{get-wood}, \texttt{get-grass}, \\
            &\texttt{use-toolshed}, \texttt{get-gold}\},\\
    \prec = \{  &\texttt{get-wood} \prec \texttt{use-toolshed}, \\
                &\texttt{get-grass} \prec \texttt{use-toolshed}, \\
                &\texttt{use-toolshed} \prec \texttt{get-gold}\}\rangle\end{aligned}$
\end{itemize}

By linearising the POPs in $\overline{\Pi}_{\texttt{gold}}$, we obtain the following sequential plans:
        \begin{itemize}
            \item[] $\begin{aligned}\pi_0 = &[\texttt{get-wood}, \texttt{get-iron},\\ &\texttt{use-factory}, \texttt{get-gold}]\end{aligned}$
            \item[]$\begin{aligned}\pi_1 = &[\texttt{get-wood}, \texttt{get-grass},\\ &\texttt{use-toolshed}, \texttt{get-gold}]\end{aligned}$
            \item[]$\begin{aligned}\pi_2 = &[\texttt{get-iron}, \texttt{get-wood},\\ &\texttt{use-factory}, \texttt{get-gold}]\end{aligned}$
            \item[]$\begin{aligned}\pi_3 = &[\texttt{get-grass}, \texttt{get-wood},\\ &\texttt{use-toolshed}, \texttt{get-gold}]\end{aligned}$
        \end{itemize}

\subsection*{Planning domain and plans for the gold-or-gem task}

The planning domain $\D_{\texttt{gold-or-gem}} = \langle \F, \A \rangle$ is the following:
\begin{itemize}
        \item[] $\begin{aligned}\F = \{ & \texttt{has-wood}, \texttt{has-grass}, \\
                & \texttt{has-iron}, \texttt{has-bridge}, \\
                & \texttt{has-stick}, \texttt{has-axe}, \\
                & \texttt{has-gold}, \texttt{has-gem} \}\end{aligned}$
        \item[] $\begin{aligned}\A = \{ & \texttt{get-wood}, \texttt{get-grass}, \texttt{get-iron}, \\
                & \texttt{use-factory}, \texttt{use-toolshed}, \\
                & \texttt{use-workbench}, \\
                & \texttt{use-toolshed-for-axe}, \\
                & \texttt{get-gold}, \texttt{get-gem}\}\end{aligned}$
\end{itemize}

where each (new) action is defined as follows:
\begin{itemize}
    \item[] $\begin{aligned}\texttt{use}&\texttt{-workbench} = \\&\langle \{ \texttt{has-wood} \}, \emptyset,\\& \{ \texttt{has-stick}\}, \{ \texttt{has-wood}\} \rangle \end{aligned}$
    \item[] $\begin{aligned}\texttt{use}&\texttt{-toolshed-for-axe} = \\&\langle \{ \texttt{has-stick}, \texttt{has-iron}\}, \emptyset,\\& \{ \texttt{has-axe}\}, \{ \texttt{has-stick}, \texttt{has-iron} \} \rangle \end{aligned}$
    \item[] $\texttt{get-gem} = \langle \{\texttt{has-axe} \}, \emptyset, \{ \texttt{has-gold}\}, \emptyset \rangle $
\end{itemize}

For the planning task, we have $\S_I = \emptyset$, and $\G = \langle \{ \texttt{has-gold}, \texttt{has-gem} \}, \emptyset \rangle$. As described in Section 4, the task is considered achieved as soon as the agent collects one between the gold ore and the gem. Notice that there are no planning actions that allow the agent to collect the gold ore and the gem at the same time. Then, finding partial-order plans for this task is equivalent to taking the union of the partial-order plans for the planning tasks $\T_{\texttt{gold}} = \langle \D_{\texttt{gold-or-gem}}, \S_I, \langle \{ \texttt{has-gold}, \emptyset \rangle\rangle $ and $\T_{\texttt{gem}} = \langle \D_{\texttt{gold-or-gem}}, \S_I, \langle \{ \texttt{has-gem}, \emptyset \rangle\rangle$. Thus, we have that the set of partial-order plans is $\overline{\Pi}_{\texttt{gold-or-gem}} = \{ \overline{\pi}_{\texttt{rope-gold}}, \overline{\pi}_{\texttt{iron-gold}}, \overline{\pi}_{\texttt{gem}}\}$, where the first two POPs are as they were for the gold task, and the third POP is as it was presented in Section 4. The only POP that adds new sequential plans compared to the gold task is $\overline{\pi}_{\texttt{gem}}$; by linearising it we obtain the following sequential plans:

    \begin{itemize}
        \item[] $\begin{aligned}\pi_4 = &[\texttt{get-wood}, \texttt{use-workbench},\\ &\texttt{get-iron}, \texttt{use-toolshed-for-axe},\\ &\texttt{get-gem}]\end{aligned}$
        \item[]$\begin{aligned}\pi_5 = &[\texttt{get-iron}, \texttt{get-wood},\\ &\texttt{use-workbench}, \\&\texttt{use-toolshed-for-axe},\texttt{get-gem}]\end{aligned}$
        \item[]$\begin{aligned}\pi_6 = &[\texttt{get-wood}, \texttt{get-iron},\\ &\texttt{use-workbench}, \\&\texttt{use-toolshed-for-axe},\texttt{get-gem}]\end{aligned}$
    \end{itemize}

\end{document}